\documentclass[11pt]{article}

\usepackage{times} 
\usepackage{helvet} 
\usepackage{courier} 
\usepackage[hyphens]{url} 
\usepackage{amsfonts,amssymb,mathrsfs,amsmath,amssymb,amsthm,float}
\usepackage{graphicx} 
\usepackage{geometry}
\usepackage{algorithm,algorithmic}
\usepackage{color}
\usepackage{newfloat}
\usepackage{listings}

\floatstyle{ruled}
\newfloat{listing}{tb}{lst}{}
\floatname{listing}{Listing}

\newtheorem{theorem}{Theorem}
\newtheorem{lemma}{Lemma}[section]
\newtheorem{cor}{Corollary}[section]

\newtheorem{remark}{Remark}[section]

\def\one{{\mathbf{1}}}

\def\I{{\mathbb{I}}}
\def\B{{\mathbb{B}}}

\def\R{{\mathbb{R}}}

\def\N{{\mathbb{N}}}
\def\L{{\mathbb{L}}}
\def\O{{\mathbb{O}}}

\def\E{{\mathcal{E}}}

\def\F{{\mathcal{F}}}

\def\D{{\mathcal{D}}}
\def\H{{\mathcal{H}}}
\def\U{{\mathcal{U}}}

\def\ND{{\mathcal{N}}}
\def\DD{{\mathcal{D}}}

\def\Loss{{\hbox{\rm{Loss}}}}
\def\Dis{{\hbox{\rm{Dis}}}}

\begin{document}

\title{\bf A Robust Classification-autoencoder\\ to Defend Outliers
and Adversaries\thanks{This work is partially supported by the NKRDP grants No.2018YFA0704705, No.2018YFA0306702 and the NSFC grant No.12288201.
}}
 \author{Lijia Yu and Xiao-Shan Gao\\
 Academy of Mathematics and Systems Science, Chinese Academy of Sciences\\
 University of  Chinese Academy of Sciences\\
 Email: xgao@mmrc.iss.ac.cn}
 \date{ }

\maketitle

\begin{abstract}
\noindent
In this paper, a robust classification-autoencoder (CAE) is proposed, which has strong ability to recognize outliers and defend adversaries.  The main idea is to change the autoencoder from an unsupervised learning model into a classifier, where the encoder is used to compress samples with different labels into disjoint compression spaces and the decoder is used to recover samples from their compression spaces. The encoder is used both as a compressed feature learner and as a classifier, and the decoder is used to decide whether the classification given by the encoder is correct by comparing the input sample with the output. Since adversary samples are seemingly inevitable for the current DNN framework, the list classifier to defend adversaries is introduced based on CAE, which outputs several labels and the corresponding samples recovered by the CAE. Extensive experimental results are used to show that the CAE achieves state of the art to recognize outliers by finding almost all outliers; the list classifier gives near lossless classification in the sense that the output list contains the correct label for almost all adversaries and the size of the output list is reasonably small.

\vskip10pt\noindent
{\bf Keywords.} Robust DNN, classification-autoencoder, list classifier,
decouple classification, outlier, adversary sample.
\end{abstract}

\section{Introduction}
The deep neural network (DNN) \cite{lecun2015deep} has become the most powerful machine learning method, which has been successfully applied in computer vision, natural language processing,
autonomous driving, and many other fields.
On the other hands, the DNN still has weaknesses for improvements, such as the
lack of explainability and robustness~\cite{choi}.

Robustness is a key desired feature for DNNs.
In general, a DNN is said to be robust, if it can not only correctly classify samples containing noises, but also has the ability to recognize outliers and to defend adversaries~\cite{sur1,sur2,outlier1,rubust-rev1}.

Outlier detection is a key issue in the open-world classification~\cite{rubust-rev1,outlier2}, where the inputs to the DNN are not necessarily satisfy the same distribution with the training data set.
On the contrary, the objects to be classified usually consist of
a low-dimensional subspace of the total input space to the DNN and the majorities
of the inputs are outliers.
To be more precise, let us consider a classification DNN $\F:\I^n\rightarrow \L$
for certain object $\O\subset\I^n$,  where $\I = [0,1]$ and $\L=\{0, 1, \ldots, o\}$ is the label set.
For the MNIST dataset, $\O$ is the hand-written numbers represented by images in $\I^{28\times }$,
$\L=\{0,1,\ldots,9\}$, and $n=784$.
In general, $\O$ is considered to be a very low-dimensional subset of $\I^n$
and hence the majorities of elements in $\I^n$ are not in $\O$,
which are called {\em outliers}.
However, for each element $x\in\I^n$, a trained  $\F$  will give a label in $\L$ to $x$,
which is wrong with high probability if
$\F$ is not specifically designed and trained to defend outliers.

A more subtle and difficult problem related to the robustness of DNN is the existence of adversaries~\cite{Biggio1,S2013}, that is, it is possible to intentionally make little modification to an image in $\O$ such that human can still recognize the object clearly, but the DNN outputs a wrong label or even any label given by the adversary.
Existence of adversary samples makes the DNN vulnerable in safety-critical applications.
%
%
Although many effective methods for training DNN to defend adversaries were proposed~\cite{sur1,sur2}, it was shown that adversaries seem still inevitable for   current DNNs~\cite{asulay1,Bast1,adv-inev1}.

There exist vast literatures on improving the robustness of DNNs~\cite{sur1,sur2,outlier1,rubust-rev1}.
In this paper, we present a new approach by changing the autoencoder from an un-supervised learning model into  a classifier.

\subsection{Contribution}

In this paper, we present a DNN which has strong ability to recognize outliers
and defend adversaries.
The basic idea  is to change the autoencoder from an unsupervised learning model into a classification network.
The encoder $\E$ of the autoencoder is used to compress the input images into a low-dimensional space $\R^m$ ($m\ll n$) such that images with the same label are compressed into  the {\em compression space} of that label
and images with different labels are compressed into disjoint subsets of $\R^m$, called compression spaces.
Furthermore, the images in $\O$ can be approximately recovered by the  decoder $\D$ from their compression spaces.
The encoder $\E$ is used both as a feature learner and a coarse classifier,
which is different from the usual classifiers in that several coordinates instead of one
are used to classify as well as to represent the images for each label.
The decoder $\D$ can be used to give the final classification
by comparing the input image with the output.

The above network is called a {\em classification-autoencoder} (CAE).
We prove that such a network exists in certain sense.
Precisely, we prove that there exists an autoencoder
such that the encoder compresses images with different labels
into disjoint compression sets of   $\R^m$ for any $m$
and the decoder  approximately recovers the input
image from its compression set with any given precision.

The CAE is evaluated in great detail using numerical experiments.
It is shown that the CAE achieves state of the art to recognize outliers by finding almost all outliers robustly.
As an autoencoder, the outputs of a CAE are always like
the object $\O$ to be classified. By definition, an outlier is an image which is not
considered to be an element of $\O$, so an image is treated as
an outlier if the input and the output are different.

The CAE also works well for adversaries in the following sense.
For a large proportion of adversaries of the encoder-classifier,
the CAE can recognize them as problem images, that is, they are outliers or adversaries.
Since adversaries are seemly inevitable for DNNs~\cite{asulay1,adv-inev1},
a possible way to alleviate the problem is to give several answers instead of one.
In a CAE, we can apply the decoder $\D$ to the compression space of each
label to recover images and output the labels and the corresponding images
which are similar to the input image.
This kind of classifier is called {\em list classifier} and the output is a {\em classification list}, which tries to give uncertain but lossless classifications.

Our experimental results show that for almost all adversaries,
the classification list contains the correct label and the  size of the output list is reasonably small.
A lossless classification  does not miss important information, which
is important for safety-critical applications. Furthermore, the classification list
can be used for further analysis.
For instance, the LCAE can  be used to do {\em decouple classification},
which means to recover one or more elements from a sample containing
more than two well-mixed elements of $\O$.

\subsection{Related work}

The autoencoder is one of the most important neural networks for unsupervised
learning~\cite{ae1}, which has many improvements and applications.
The autoencoder learns compressed features in a low-dimensional space for
high-dimensional data with minimum reconstruction loss, while our CAE makes classification
at the same time of learning features.
It is natural to use autoencoders for outlier detection due to its reconstruction property~\cite{ae3}.
We improve this in two aspects. First, by compressing images with different labels
into disjoint compression spaces, the robustness is increased and the classification can be given.
Second, we use the compression spaces to introduce the list classification to increase the robustness.
The ladder network, which is a variant of autoencoder,
was used to classify the input images~\cite{ladder}.
Our work is different in two aspects.
First, a different DNN structure is introduced in this paper
and thus the loss functions are different.
Second, the work in \cite{ladder} was mainly focused on denoising
and our work is mainly for defending adversaries and outliers.
The robustness of autoencoder was studied~\cite{ae2,ae3,ae4}.
In principle, these methods can be applied to our CAE model to further
improve the robustness.

A simple approach to recognize outliers is to introduce a new label
representing outliers and add outlier samples~\cite{rubust-rev1}.
The difficulty with this approach is that  the distribution of outliers is usually
too complex to model in high dimensional spaces.
As a consequence, a network based on this approach works well for those outliers similar to that in the training set and works poorly for other outliers, as shown by the experimental results in this paper.
Many approaches were proposed to detect outliers~\cite{outlier2,outlier1,rubust-rev1}.
On the other hand, the CAE proposed in this paper is more natural
to detect outliers, because the output of the CAE are assumed to be
similar to the objects to be classified.

Many methods were proposed to train more robust DNNs to defend adversaries~\cite{sur-adv}.
In the adversary training method proposed by Madry et al,
the value of the loss function at the worst adversary in a small neighborhood of
the training sample is minimized~\cite{M2017}.
This approach can reduce the adversaries significantly~\cite{M2017,YOPO}.
A similar approach is to generate adversaries and add them to the training set~\cite{G2014}.
Another major approach is the gradient obfuscation methods,
which deliberately hide or randomize the gradient information
of the model, so that gradient based attacks cannot be used directly~\cite{obfu1,obfu2,obfu3,obfu4,obfu5,bias}.
The ensembler adversarial training~\cite{adv-li2} was introduced
for  CNN models, which can apply to large datasets
such as ImageNet.
A fast adversarial training algorithm was proposed,
which improves the training efficiency by reusing the backward pass calculations~\cite{adv-li3}.
A less direct approach to enhance the ability for
the network to resist adversaries is to make the DNN more stable
by introducing the Lipschitz constant or $L_{2,\infty}$ regulations of each layer~\cite{adv-li1,S2013,L2i}.

Effective methods were proposed to train more robust DNNs to defend adversaries~\cite{sur-adv}.
In the adversary training method proposed by Madry et al,
the loss function is a min-max optimization problem
such that the value of the loss function of the worst adversary in a small neighborhood of
the training sample is minimized~\cite{M2017}.
This approach can reduce the adversaries significantly~\cite{M2017,YOPO}.
Another similar approach is to generate adversaries and
add them to the training set~\cite{G2014}.
The ensembler adversarial training~\cite{adv-li2} was introduced
for  CNN models, which can apply to large datasets
such as ImageNet.
A fast adversarial training algorithm was proposed,
which improves the training efficiency by reusing the backward pass calculations~\cite{adv-li3}.
A less direct approach to enhance the ability for
the network to resist adversaries is to make the DNN more stable
by introducing the Lipschitz constant or $L_{2,\infty}$ regulations of each layer~\cite{adv-li1,S2013,L2i}.

Increase the robustness of the network in general will increase its ability
to defend adversaries and there exist quite a lot of work on
robust DNNs~\cite{H2015,rubust-rev1}.
Adding noises to the training data is an effective way to increase
the robustness~\cite[Sec.7.5]{DL}.
The $L_{1}$ regulation and $L_{1,\infty}$ normalization are used to increase the robustness of DNNs~\cite{Y2018}.
Knowledge distilling is also used to enhance  robustness and defend adversarial examples~\cite{H2015}.
In \cite{rob-b3,W2019,rob-b1,rob-b2}, methods to compute the robust regions of a DNN were given.
%


\vskip10pt
The rest of this paper is divided into four parts.
In section \ref{sec-s}, we  give the structure for the CAE and prove its existence.
In section \ref{sec-c}, we  give experimental results
and show that the CAE can recognize almost all outliers.
In section \ref{sec-lcae}, we give the list classification algorithm and the experimental results for it to defend adversaries.
In section \ref{sec-conc}, conclusions are given.

\section{Structure and existence of the classification-autoencoder}
\label{sec-s}

\subsection{The main idea}

Let $\I = [0,1]$ and $\F$ an autoencoder
\begin{eqnarray}\label{eq-F1}
\F=\D\circ\E &:& \I^n\rightarrow \I^n
\end{eqnarray}
with   encoder $\E: \I^n\rightarrow \R^m$ and   decoder $\D: \R^m\rightarrow \I^n$
($m\ll n$).
$\F$  is called a {\em classification-autoencoder (CAE)} for a classification problem,
if $\E$ compresses an input sample with different labels into disjoint {\em compression spaces} in $\R^m$ and  $\D$ recovers the input images to any given precision from their compression spaces.
%

\begin{figure}[H]
\centering
\includegraphics[scale=0.35]{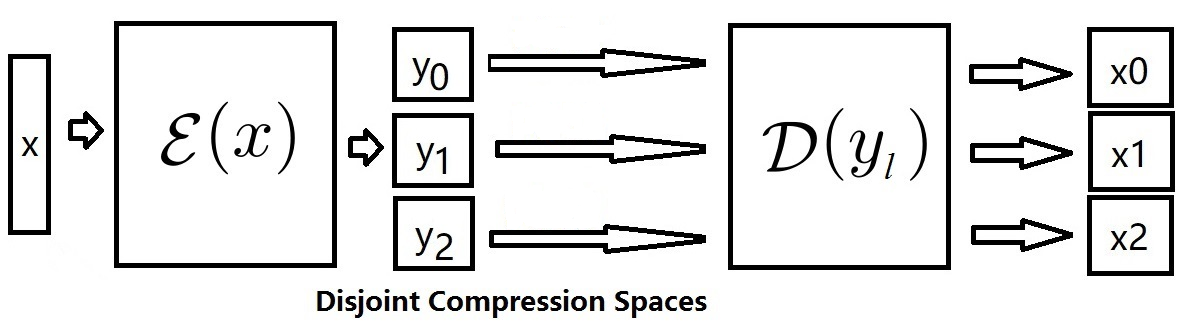}
\caption{CAE: The input $x$ is mapped into $y_l$ in disjoint compression spaces by the encoder $\E$, where $l$
is the label of $x$ given by $\E$.
The decoder $\D$ recovers $x_l$ from $y_l$, which is approximately the same as $x$ if $l$ is the correct label of $x$.}
\label{fig-cae}
\end{figure}

As shown in Figure \ref{fig-cae}, for an input image $x$,
let $y_l$ be the projection of $\E(x)$ to the compression space of label $l$.
Then the encoder gives a potential label $l$ to $x$ if $y_l$ has larger weight than other $y_i$ for $i\ne l$. Furthermore, if $\F(x)$ is very similar to $x$, then we believe that $x$ is a normal sample with label $l$.
On the contrary, if $\F(x)$ not similar to $x$, then $x$ is either an outlier or an adversary.
In this way, the CAE can be used as an open-world classifier by doing
classification for normal samples as well as recognizing outliers and adversaries.

\subsection{Structure and training of the classification-autoencoder}

\subsubsection{The classification-encoder $\E$}
The encoder $\E$ will be used both as a compressor as well as a classifier.
Let $\L=\{0, 1, \ldots, o\}\subset\N$ be
the set of labels and $m=(o+1)m_0$,
where $m_0\in\N_{+}$ is a hyperparameter to be defined by the user.
We try to use $\E$ to compress an image with label $l$ to
the {\em compression space} of  label $l$,
which consists of the $k$-th coordinates of $\R^m$ for $k=lm_0+1, \ldots, lm_0+m_0$.
%
%
For $l\in\L$, define the {\em compression mask vectors} $M_l\in\R^m$ for the $l$-th compression space as follows
\begin{equation}
\label{eq-M}
 M_l[j]=
 \left\{
 \begin{array}{ll}
        1,  & \text{when } lm_0+1\le j \le lm_0+m_0\\
        0,  & \text{otherwise.}
  \end{array}
 \right.
\end{equation}
For $y\in \R^m$ and $l\in\L$,  the {\em projection weight} of $y$ to the $l$-th compression space is
\begin{equation}
\label{eq-W}
 W_l(y)=y\cdot M_l
\end{equation}
where $\cdot$ is the inner product.
%

Let $x\in\I^n$ be a sample in the training set and $l_x\in\L$ its label.
In order to send $x$ to its compression space,  we define the following loss function for $\E$
\begin{equation}
\label{eq-L1}
\Loss_{\E}(x, l_x)=\Loss_{\rm CE}(A(x)/\gamma, l_x)
\end{equation}
where $A(x)=\{W_0(\E(x)),$ $\ldots,$ $ W_o(\E(x))\}$ and  $\gamma$ is a hyperparameter.
Intuitively, the loss function means that the weight of $\E(x)$ in the compression
space for label $l_x$ will be maximized.

%

\subsubsection{The classification-autoencoder}
The decoder $\D$ in \eqref{eq-F1} tries to recover $x$ from its compression subspace for an $x\in\I^n$.
We first define a projection function
$$p(y, l)=y\circ M_l:\R^m\times\L\to \R^m$$
where $\circ$ is the Hadamard (element-wise) product
and $M_l$ is the mask vector defined in \eqref{eq-M}.
Note that $p(y, l)$ projects $y$ to the compression space of label $l$.

Then the total network $\F$ is certain composition of $\E$ and $\D$:
\begin{equation}\label{eq-F2}
\F(x) = \D(p(\E(x), l_x)) = \D(\E(x)\circ M_{l_x})
\end{equation}
and the loss function for $\F$ is
\begin{equation}\label{eq-loss}
\Loss(x, l_x)=\Loss_{\E}(x, l_x)+\lambda \lambda\Loss_{\rm{MSE}}(\F(x), x)
\end{equation}
where  $\lambda$ is a hyperparameter.

Moreover, to increase the robustness of the CAE, we can use adversarial training~\cite{M2017}
and the loss function for $\F$ becomes
\begin{equation}
\label{eq-loss-at}
\Loss(x, l_x)=\Loss_{\E}(x+\eta_x, l_x)+\lambda(\Loss_{\rm{MSE}}(\F(x+\eta_x), x)+\Loss_{\rm{MSE}}(\F(x), x))
\end{equation}
where $\eta_x=\arg\max_{||\eta||_{\infty}<\epsilon}\Loss_{\E}(x+\eta, l_x)$.

\begin{remark}
\label{rem-20}
The loss function is a joint loss over
the classification space as well as
a reconstruction loss on the decoder phase,
which makes $\F$ both as a classifier and as a feature learner.
The CAE is different from the usual classifiers in that several coordinates instead of one
are used to classify as well as to learn features.
\end{remark}

Given a set  of training samples,  we can train $\F$
with the standard gradient descendent method
based on the above loss functions.
In what below, we give a termination criterion for the training algorithm.
Since a sample $x$ with label $l_x$ should satisfy $C(x)=1$,  we have $\E(x)\cdot M_{l_x}-\E(x)\cdot M_i\ge C_0$ for $i\neq l_x$. Thus
$$\Loss_{\E}(x,l_x)
=-\ln{\frac{e^{\E(x)\cdot M_{l_x}/\gamma}}{\sum_{i=0}^{o}e^{\E(x)\cdot M_i/\gamma}}}
=-\ln{\frac{1}{\sum_{i=0}^{o}e^{-(\E(x)\cdot M_{l_x}-\E(x)\cdot M_i)/\gamma}}}
\le\ln{(1+9e^{-C_0/\gamma})}.$$
So  when $\Loss_{\E}(x, l_x)\le\ln{(1+9e^{-C_0/\gamma})}$ for most $(x, l_x)$ in the training set and $\Loss(x, l_x)$ is small enough,  we terminate the algorithm.

\subsection{The open-world classification algorithm}
Let $\F$ be a trained CAE with the loss functions \eqref{eq-loss} or \eqref{eq-loss-at}.
In this section, we show how to use $\F$ as an open-world classifier.
Precisely, the algorithm will return a label in $\L$
if the input is considered to be a normal sample and
return label $-1$ if the input is considered to be a {\em problem image}.
A problem image could be either an outlier or an adversary.

For an input $x\in\I^n$, according to the loss function in \eqref{eq-L1}, the {\em pseudo-label} of $x$ is
\begin{equation}
\label{eq-L}
L(x)=\arg\max_{l\in\L} W_l(\E(x)):\R^n\rightarrow \L.
\end{equation}
With the pseudo-label $L(x)$, the output of the network is defined as
\begin{equation}
\label{eq-F23}
\widehat{{\F}}(x) = \D(p(\E(x), L(x))) = \D(\E(x)\circ M_{L(x)}).
\end{equation}
The main idea of the algorithm is to compare $x$ and
$\widehat{{\F}}(x)$ to see if $x$ is a normal sample or a problem image.
%
%
%

\begin{algorithm}[H]
\caption{CAE}
\label{alg1}
Input: $x\in\I^n$, hyperparameters $C_0, b_s, b_u$ in $\R_{+}$.

Output: a label of $x$ in $\L$ or label $-1$ meaning that $x$ is a problem image.
\begin{description}
\item[S1]
Compute $\F(x)$ with \eqref{eq-F2}, $L(x)$ with \eqref{eq-L}, and $\widehat{\F}(x)$ with \ref{eq-F23}.
%

\item[S2]
If $||\widehat{\F}(x)-x||\le b_s$ for a given threshold $b_s\in\R_{+}$,
then output: label $L(x)$.

\item[S3]
If $||\F(x)-x||\ge b_u$ for a given threshold $b_u\in\R_{+} (b_s<b_u)$,
then output label $-1$.

\item[S4]
This step treats the $x$ satisfying
$b_s<||\widehat{\F}(x)-x||< b_u$.
\begin{description}
\item[S4.1]
Let $d_l=||\D(p(\E(x), l))-x||$,  for all $l\in\L$.

\item[S4.2]
If $d_{L(x)}\le d_l$ for  $l\in \L$,  then output label $L(x)$.

\item[S4.3]
Output label $-1$.
\end{description}
\end{description}
\end{algorithm}

We explain the algorithm as follows.
Let $\O$ be the set of images to be classified.
%
In Step S2, when $\widehat{\F}(x)$ and $x$ are similar enough, we think $x$ is in $\O$ and its label is given.
In Step S3, when $x$ is not anything like $\widehat{\F}(x)$, $x$ is a problem image.
In Step S4, we are not sure whether $x$ is an element in $\O$ or a problem image.
In this case, we give a refined analysis by computing  $\D(p(\E(x)),l)$ for all $l\in\L$
and checking if $\D(p(\E(x)),L(x))$ is more similar to $x$
than $\D(p(\E(x)),l)$.
%

\begin{remark}
\label{rem-21}
The input to Algorithm \ref{alg1} could be a sample in $\O$ or a problem image.
By the {\em accuracy} of the algorithm on an input set, we use the usual meaning,
that is, the percentage of samples which are given the correct labels in $\L$.
By the {\em total accuracy} of the algorithm on an input set, we mean
the percentage of inputs which are in $\O$ and are given the correct labels
and the inputs which are problem images and are considered as problem images.
\end{remark}

\subsection{Existence of the CAE}
\label{sec-e}

The main idea of CAE is that images with the same label are mapped into a compression space
and images with different labels are compressed into disjoint compression spaces by the encoder.
Furthermore, the images to be classified can be approximately recovered from their compression spaces.
In this section, we prove that DNNs with such properties exist in certain sense.

Let $x\in\R^{n}$ and $r\in\R_{+}$. When $r$ is small enough,  all images in
$$\B(x,r) =\{x+ \eta\,|\, \eta\in\R^n, ||\eta||< r\}$$
can be considered to have the same label with $x$.
%
Therefore, the set of images $\O$ to be classified can be considered
as bounded open sets in $\R^n$.
%
This observation motivates the existence result given below,
where  $V(S)$ is used to denote the volume of $S\subset \R^n$.

\begin{theorem}
\label{th-3}
Suppose that the objects to be classified consist of a finite number of
disjoint and bounded open sets $S_l\subset\I^n$  such that
objects in $S_l$ have label   $l\in\L=\{0,\ldots,o\}$.
Then, for any $m \in\N_{+}$ and for any $\epsilon,\gamma$ in $\R_+$, there exist DNNs
$\E:\I^n\to\R^m$ and $\D:\R^m\to \I^n$  such that
\begin{description}
\item[1.]
Let $K=\{x\in\bigcup_{l=0}^{o}S_l \,|\,  \E(x)\in\E(S_l)\cap\E(S_k)$ $\hbox{ for some } l\ne k\}$. Then $\frac{V(K)}{\sum_{i}V(S_i)}<\epsilon$.
\item[2.] Let $\widehat{S}_{l}= \{x\in S_l \,|\,
   ||\D(\E(x))-x||<\gamma\}$ for   $l\in\L$. Then  $\frac{V(\widehat{S}_{l})}{V(S_{l})}>1-\epsilon$.
\end{description}
\end{theorem}
\begin{proof}
We give the idea of the proof and the detail of the proof is given in Appendix A.
%
%
Let $\I_0=(0,1)$, $A_0=\I_0^m$, and $A_j=\{y\,|\,y=z+2j\one, z\in A_{j-1}\}$ for $j=1,\dots,o$, where $\one\in \R^m$ is the vector whose coordinates are $1$.
Then we can define piece-wise constant functions
${E}:\I^n\to\R^m$ and ${D}:\R^m\to \I^n$,
such that ${E}$ is approximately the identity map from $S_l$ to $A_l$
with $D$ the inverse map for $l\in\L$,
that is $A_l$ is the compression space of label $l$.
Finally, based on the universal approximation property of DNNs,
the required autoencoder can be constructed.
\end{proof}

\begin{remark}
Theorem \ref{th-3} implies that different $S_l$ are mapped
into approximately disjoint spaces $\E(S_l)\cong A_l$
and almost all elements in $S_l$ can be approximately
recovered from $\E(S_l)$ by $\D$, and thus proves the existence of the CAE.
\end{remark}

\begin{remark}
In Theorem \ref{th-3}, $m$ could be as small as possible,
and $\E(S_l)$ are contained in disjoint unit cubes in $\R^m$.
In practice, in order that the network can be trained efficiently,
we assume $\E(S_l)\subset \R^{m_l}$ for some $m_l\in \N_+$,  $m=\sum_{=0}^o m_l$, and $\R^m$
the direct product of these $\R^{m_l}$.
\end{remark}

\section{Experimental results for CAE}
\label{sec-c}
In this section,  we will give  experimental results based on MNIST dataset in great detail.
The precise structures of the networks used in the experiments
are given in the Appendix B.
The hyperparameters are $\gamma=50$, $\lambda=1$, $C_0=80$,  $b_l=0.04$,  and $b_u=0.09$.
The codes can be found at
\newline https://github.com/yyyylllj/NetB.

\subsection{Accuracy on the test set}
\label{sec-c1}

We first show that the CAE works well for the test set as a classifier.
Among the 10000 samples in the test set of MNIST,
9864 samples are given the correct labels,
43 samples are given the wrong labels,
and
93 samples are considered as problem images.
Some of the samples which are give the wrong label or considered as problem images are given in Figure \ref{fig-202-1}.

\begin{figure}[H]
\centering
\includegraphics[scale=0.55]{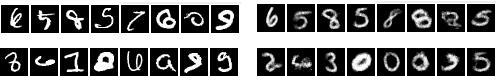}
\caption{The images in left-side part are the input images and
the ones in the right-side part are the corresponding output of CAE.
Row 1: problem images. Row 2: images for which wrong labels are given.}
\label{fig-202-1}
\end{figure}

\begin{table}[H]
\centering
\begin{tabular}{|c|c|c|c|}
  \hline
Data Set &  CAE & LCAE  & NLabels\\
  \hline
MNIST  &98.64$\%$ & 99.97$\%$ & 1.72\\
  \hline
\end{tabular}
\caption{Accuracies for CAE and and LCAE (section \ref{sec-lcae2}).
NLabels is average number of labels in the output list of LCAE.
 }
\label{tab-ac1}
\end{table}

\begin{remark}
As mentioned above, 39 samples are given the wrong labels
and 131 samples are considered as problem images.
As we can see in Figure \ref{fig-202-1},
the quality of these images is quite poor
and they indeed can be considered as problem images.
This property is not all bad and can even be used to identify bad samples in a data set.
A solution to recognize almost all of these problem images is given in section \ref{sec-lcae}
and the result is given in the column LCAE in Table \ref{tab-ac1}.
\end{remark}

\subsection{Recognize outliers}
\label{sec-out}

In this section, we use four types of outliers to check the ability of $\F$ to recognize outliers.

The type 1 outliers are images in $\R^{28\times 28}$,
whose coordinates are generated by $\DD_n$ (iid).
Here, $\DD_n=1/2(\ND+\U)$,  where $\ND$ is the normal distribution and $\U$ is the uniform distribution.

The type 2 outliers are random samples with structures.
They are created by randomly generating a row, a column, or the diagonal, and then by copying the row, column or the diagonal randomly to cover the matrix.

For the type 3 outliers, the middle ($12\times12$) 144 pixels of an image in $\R^{28\times 28}$  are given by the normal distribution $\ND$ (iid).

%
The type 4 outliers for MNIST are generated by reducing the sizes of images in CIFAR-10 to $28\times28$ and change the image from color to grey.
For CIFAR-10, type 4 outliers are obtained from MNIST similarly.


\begin{figure}[H]
\centering
\hspace{2mm}
\includegraphics[scale=0.45]{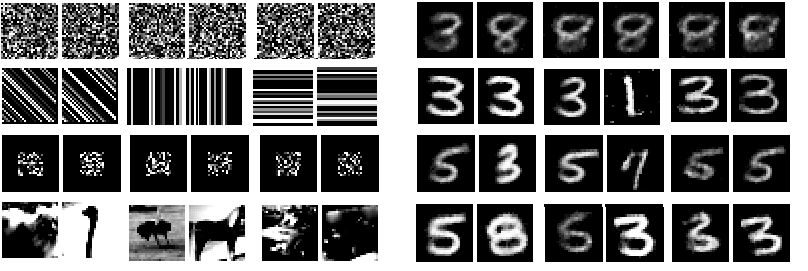}
\caption{The four types of outliers (left-side) for MNIST and theirs outputs (right-side) from $\F$ are given in the four rows from type 1 to type 4.} 
\label{fig-mnc}
\end{figure}

\begin{figure}[H]
\centering
\includegraphics[scale=0.70]{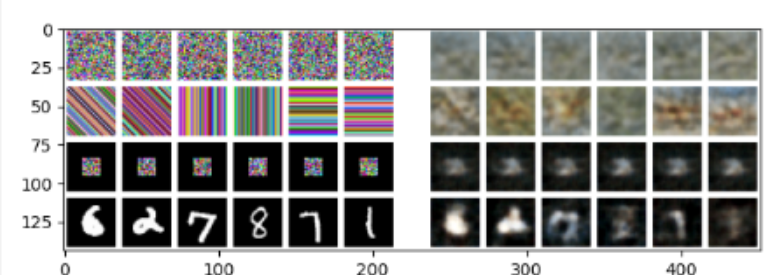}
\caption{Four types of outliers  (left-side) for CIFAR-10 and   outputs (right-side). } 
\label{fig-C10c}
\end{figure}

We compare our network with a CNN $\H$ which is trained with MNIST plus 60000 noise samples with a new label $-1$, representing the class of outliers.
The noise samples  are generated by  $\alpha\ND+\beta\U$,
where  $\alpha$ is a random number in $(0,1)$ and $\beta=1-\alpha$.

We compare $\F$ and $\H$  and the results are given in
the first four rows of Table \ref{tab-co1}.
%
%
As shown in Figure \ref{fig-mnc}, for an outlier $x$, $x$ and $\F(x)$  are quite different, which means $||\F(x)-x||$ is  big  enough to make the network treating  $x$ as a problem image.
As a consequence, our network can recognize all kind of outliers.
On the other hand, $\H$ can only recognize the outliers similar to the training outlier samples.
%
%

We also compare the robustness of $\F$ and $\H$ to recognize outliers.
Let $x_o$ be a type 1 outlier.
Two types of new outliers are generated from $x_o$ as follows.
For $\H$, we use gradient descent for $x_o$ to make $L_{\rm{CE}}(x_o, -1)$ smaller. For $\F$, we use gradient descent for $x_o$ to make
$L_{\rm{MSE}}(\D(p(\E(x_o),L(x_o))-x_o))$ smaller.
Use Type 1.i to denote the new outliers, where each pixel of $x_o$ is changed up to 0.i for $i=1,2$ and the results are listed in the fifth and sixth rows of Table \ref{tab-co1}.
From Table \ref{tab-co1}, we can see that $\F$ is very robust to recognize
this kind of strong outliers, while $\H$ is much less robust.
 
\begin{table}[H]
\centering
\begin{tabular}{|c|c|c|c|c|}
  \hline
  Outliers & CAE(M) & Network $\H$(M) &CAE(C)&Network $\H$(C)\\\hline
Type 1  & 100$\%$ & 99$\%$ & 100$\%$ & 99$\%$\\
Type 2  & 100$\%$ & 97$\%$ & 100$\%$ & 92$\%$\\
Type 3  & 100$\%$ & 3$\%$ & 100$\%$ & 23$\%$\\
Type 4  & 100$\%$ & 26$\%$& 100$\%$ & 40$\%$\\
Type 1.1  & 100$\%$ & 80$\%$& 100$\%$ & 10$\%$\\
Type 1.2  & 100$\%$ & 20$\%$& 100$\%$ & 5$\%$\\
\hline
\end{tabular}
\vskip3pt
\caption{Percentages for CAE and $\H$ to recognize  outliers or the total accuracy defined in Remark  \ref{rem-21},(C) means experiment in CIFAR-10, (M) means MNIST.}
\label{tab-co1}
\end{table}

\begin{remark}
Types 1, 2, 3 are all outliers. We use them to show that
a network trained with samples outliers works well for
samples similar to the training set, but works poorly
for samples different from the training set.
For instance, $\H$ works well for type 1 and type 2 outliers,
but it works poor for type 3 outliers.
The type 5 adversarial outliers are also used to show that
there exist no exact boundaries between outliers and normal samples satisfying the distribution
and outliers.
\end{remark}

To summarize, the CAE can recognize almost all outliers robustly
and this is one of the main advantages of the network introduced in this paper.

\subsection{Defend adversaries}
\label{sec-ad}
In this section, we check the ability of $\F$ to defend adversaries.
For $\F$, $\E(x)$ is the network for classification, so we create adversaries for $\E(x)$.
We use  samples in the training set to generate 3 types of adversaries.

The type 1.i ($i=10,20,30$) adversaries ($L_{\inf}$ adversary)
are created with PGD-$i$~\cite{M2017},
where each step changes 0.01 for every pixel and at most $i$ steps are used.

The type 2.i ($i=40,60,80$) adversaries ($L_0$ adversary)
are obtained with JSMA ~\cite{P2016} by changing $i$ coordinates of a training sample, and each coordinate can change at most 1.

The type $3.i$ ($i=0,1,2$) adversaries are called {\em strong adversaries}
which are generated by two steps.
First, generate a type $1.20$ adversary $x_a$.
Second, use gradient descent on $x_a$ to make $L_{\rm{CE}}(\F(x_a),l)$ bigger
and the change for each pixel is $\le0.i$.
%

\begin{figure}[H]
\centering
\hspace{2mm}
\includegraphics[scale=0.75]{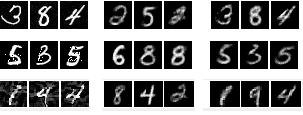}
\caption{Three types adversaries (left-side) and their outputs from $\F$ (middle) are given in the three rows.
The right-side: output of $\D(g(\E(x), l_x))$,  $l_x$ is the label of $x$. }
\label{fig-a22}
\end{figure}

We will compare with the well-known adversarial training method proposed in \cite{M2017}
for a CNN whose structure is given in Appendix B.
The results in Table \ref{tab-ca2} are the adversarial creation rates  from the test set.
The adversarial creation rate for type 1.10 adversaries by $\E$ is about $35\%$.
We use 1000 type 1 adversaries  as input to $\F$.
Among the 1000 inputs,
733 samples are considered as problem images,
267 samples are given wrong labels.
So,  for about $26.7\%$ of the type 1 adversaries of $\E$,
$\F$ gives  wrong labels and  $73.3\%$ of them are considered as problem images.
Therefore, the adversary creation rate for type 1 adversaries is $9.35\%=35\% \cdot 26.7\%$.
%
%
%
%
Results for other types are computed similarly.

\begin{table}[H]
\centering
\begin{tabular}{|l|r|r|r|c|}
  \hline
Attack &  AT & CAE & LCAE & NLabels\\
  \hline
Type  1.10  & 6$\%$  &3$\%$   &$<1\%$ &3.04\\
Type  1.20  & 15$\%$ &15$\%$  &$<1\%$& 3.10\\
Type  1.30  & 55$\%$ &22$\%$ & $1\%$&3.13\\
Type  2.40  & 55$\%$ &19$\%$ & $<1\%$&3.58\\
Type  2.60  & 76$\%$ &23$\%$ & $1\%$&3.64\\
Type  2.80  & 85$\%$ &30$\%$  & $2\%$&3.70\\
Type  3.0 & 2$\%$    &1$\%$  &$<$1$\%$&2.87\\
Type  3.1 &  18$\%$  &11$\%$ & 1$\%$&3.22\\
Type  3.2 &  51$\%$  &34$\%$ & 4$\%$&3.94\\

  \hline
\end{tabular}
\caption{Adversarial creation rates for CAE (second column) and LCAE
(third column; see section \ref{sec-lcae3}).
Attack is the attack methods: we use PGD-$i$ for type 1$.i$ adversaries
and JSMA for type 2$.i$ adversaries.
AT is the results for the network trained with adversarial training~\cite{M2017}.
}
\label{tab-ca2}
\end{table}

From Table \ref{tab-ca2},
we have the following observations.
(1)  CAE is always better than the adversarial training, and much better for more difficult adversaries.
The performance of CAE is more stable than that the adversarial training,
whose adversarial creation rates are less than 30$\%$.
For adversarial training, the adversarial creation rate for type 2.80 attack method is 85$\%$.
(2) The results for $\F$ and the adversarial training are different.
Since the input samples are adversaries of $\E$, the CAE cannot give
the correct label for these adversaries and the adversarial creation rate for CAE
is the percentage of inputs for which CAE give wrong labels.
%
%
%
%
We can see that CAE also performs better for strong adversaries.

\section{List classifier to defend adversaries}
\label{sec-lcae}

It was widely believed that adversaries are inevitable for the current DNN framework~\cite{asulay1,Bast1,adv-inev1}.
A possible way to alleviate the problem is to give several
labels instead of one. In this section, we give such an approach
based on the CAE.

\subsection{The list classifier LCAE}
\label{sec-lcae1}

The list classification algorithm is motivated by Figure \ref{fig-a22}, where the label $L(x)$ is wrong, but
$\D(g(\E(x), l_x))$ gives a very close approximation to the input $x$.
The idea is to output all labels $l$ such that  $\D(g(\E(x), l))$ is similar to $x$ with ``high probabilities''.

We first define a distance between two images.
Let $x\in \R^n$, and $A(x)$  the average of all coordinates of $x$.
For $a\in \R$, define $S(a)=1$ if $a>0$, and $S(a)=0$ if $a\le0$.
Define $\widehat{x}=(\widehat{x}_1,\ldots,\widehat{x}_n)\in \{0,1\}^n$,
where  $\widehat{x}_i=S(x_i-A(x))$.
We call $\widehat{x}$ the standardization of $x$.
Define the distance between $x$ and $y$ as $\Dis(x,y)=\frac{A(x-y)}{A(x+y)}$ for $x,y\in \R^n$.

Suppose that $\F$ in \eqref{eq-F1} is a trained CAE with $\L$ as the label set.
We do image classification by giving several possible answers.
\begin{algorithm}[H]
\caption{LCAE}
\label{alg2}
The input: $x\in\I^n$, hyperparameters $B$ and $D$ (in the experiments below, we choose B=14 and D=$10\%$).

The output: a list of labels and the corresponding  images.
\begin{description}
\item[S1]
Compute $\F(x)$ with \eqref{eq-F2}  and $L(x)$ with \eqref{eq-L}.

\item[S2]
Let $x_l=\D(g(\E(x), l))$,  for all $l\in\L$.

\item[S3]
Compute $L_l=||x_l-x||$, $D_l=\Dis(\widehat{x}_l,\widehat{x})$, and $E_l=D_lL_l$
for $l\in\L$.

\item[S4]
The probability for $x$ to be an outlier is $P_{-1}=\frac{1-e^{-(BE_m)^2}}{1+e^{-(BE_m)^2}}$, where $E_{m}=\min_l E_l$.

\item[S5]
The probability for $x$ to have label $l$  is
 ${P}_l=(1-P_{-1})\bar{P}_l$,
where $\bar{P}_l=\frac{1/E_l}{\sum_j1/E_j}$ and
 $l\in\L$.

\item[S6]
Output $(l,x_l)$ if ${P}_l> D$ for $l\in\L$.

\end{description}
\end{algorithm}
 We give an illustrative example.
In Figure \ref{fig-G1}, the images $x,x_l, \widehat{x},\widehat{x}_l$ are given.
In Table \ref{tab-Gt1}, the probabilities $P_l$ are given.
The  probability of the picture to have label $1$ is $88.80\%$
and all other probabilities are smaller than $0.1$.

\begin{figure}[H]
\centering
\includegraphics[scale=0.6]{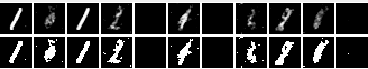}
\caption{The first row is $x$ and $x_l$ in S2 of Algorithm \ref{alg2} and
the second row is $\widehat{x}$ and $\widehat{x}_l$, for $l=0,\ldots,9$.}
\label{fig-G1}
\end{figure}

\begin{table}[H]\scriptsize
\centering
\begin{tabular}{|c|c|c|c|c|c|c|c|c|c|c|c|}
\hline
%
$l=$ & $-1$ & $0$ & $1$ & $2$ & $3$ & $4$  & $5$ & $6$ & $7$ & $8$ & $9$ \\\hline
 $\Dis(x,x_l)$
 && 0.0473 & 0.0029 & 0.0290 & 0.0727 & 0.0266 & 0.0727  & 0.0515 & 0.0255 & 0.0357 & 0.0727 \\
 $\Dis(\widehat{x},\widehat{x}_l)$
 && 0.3848 & 0.069  & 0.2759 & 1      & 0.2820 & 1       & 0.4951 & 0.2902 & 0.2773 & 1 \\
$E_l$ &  &  0.0182&0.0002  &0.0080  &0.0727  &0.0075  &0.0727  &0.0255  &0.0074  &0.0099  &0.0727\\
$\bar{P}_l$&&0.0098&0.8881& 0.0222& 0.0024&0.0237& 0.0024& 0.0070& 0.0240& 0.0179&    0.0024\\
${P}_l$&$0.01\%$&0.98$\%$&88.80$\%$& 2.22$\%$& 0.24$\%$&2.37$\%$& 0.24$\%$& 0.70$\%$& 2.40$\%$& 1.79$\%$&    0.24$\%$\\
\hline
\end{tabular}
\caption{The probabilities for $x$ to have label $l$}
\label{tab-Gt1}
\end{table}

In the rest of this section, we give experimental results for LCAE.
The structure and hyperparameters of the CAE are the same as that
used in the preceding section \ref{sec-c}.
The codes can be found at
\newline
https://github.com/yyyylllj/NetB.

\subsection{Accuracy on the test set}
\label{sec-lcae2}
%
We use LCAE to the test set of MNIST, which contains 10000 samples.
%
%
%
The number of samples whose output contains the correct label is 9997.
The number of samples whose output only contains the correct label is 5712.
The number of samples whose output contains outlier is $1$.
In Table \ref{tab-lc1}, we give the numbers of labels in the output lists.
The average number of labels in the output list is $1.722$.

Comparing to the result in Table \ref{tab-ac1},
at the cost of outputting a list containing $1.722$ labels,
the accuracy could be increased from 98.64$\%$ to
99.97$\%$.

\begin{table}[H]
\centering
\begin{tabular}{|c|c|c|c|c|c|c|}
  \hline
Number of labels & 1 & 2 & 3 & 4 & 5 & 6 \\
Number of samples  & 5712 & 2122 & 1500 & 560 & 104 & 2 \\
  \hline
\end{tabular}
\caption{Number of labels in the output}
\label{tab-lc1}
\end{table}

In Figure \ref{fig-G2}, we give the unique sample
whose output contains the outlier.
In Figures \ref{fig-G3}, \ref{fig-G3-bc}, \ref{fig-G3-bc1},  we give the three samples whose outputs do not contain the correct label.
In Table \ref{Gt2}, we give the probabilities of these figures.

\begin{figure}[H]
\centering
\hspace{2mm}
\includegraphics[scale=0.55]{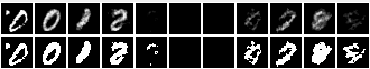}
\caption{Output contains the correct label $0$ and the outlier.
The first row is $x$ and $x_l$ in S2 of Algorithm \ref{alg2} and
the second row is $\widehat{x}$ and $\widehat{x}_l$, for $l=0,\ldots,9$.
}
\label{fig-G2}
\end{figure}

\begin{figure}[H]
\centering
\hspace{2mm}
\includegraphics[scale=0.55]{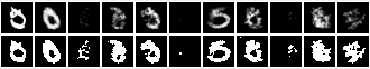}
\caption{Output contains labels $0$, $3$, $6$, but does not contain the correct label $5$.}
\label{fig-G3}
\end{figure}
\begin{figure}[H]
\centering
\hspace{2mm}
\includegraphics[scale=0.62]{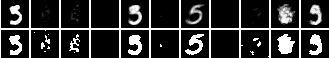}
\caption{Output contains labels $3$ and $9$, but does not contain the correct label $5$.}
\label{fig-G3-bc}
\end{figure}
\begin{figure}[H]
\centering
\hspace{2mm}
\includegraphics[scale=0.62]{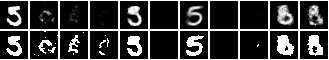}
\caption{Output contains label $3$, but does not contain the correct label $5$.}
\label{fig-G3-bc1}
\end{figure}

\begin{table}[H]\footnotesize
\centering
\begin{tabular}{|c|c|c|c|c|c|c|c|c|c|c|c|c|}
\hline
& $l=$ & $-1$ & $0$ & $1$ & $2$ & $3$ & $4$  & $5$ & $6$ & $7$ & $8$ & $9$ \\\hline
Figure \ref{fig-G2}&
 $P_l$& $12.5\%$&13.1$\%$& 11.7$\%$& 12.9$\%$& 4.52$\%$& 4.20$\%$& 4.20$\%$& 9.39$\%$& 8.50$\%$& 11.7$\%$&    7.21$\%$\\
 Figure \ref{fig-G3}&
$P_l$&0.25$\%$&36.2$\%$& 1.91$\%$& 6.94$\%$& 14.7$\%$& 1.25$\%$& 8.27$\%$& 12.4$\%$& 1.34$\%$& 12.5$\%$&  4.08$\%$\\
Figure \ref{fig-G3-bc}&
$P_l$&0.01$\%$&0.63$\%$& 0.75$\%$& 0.43$\%$&59.5$\%$& 0.43$\%$& 4.68$\%$& 0.43$\%$&0.65$\%$& 2.64$\%$&     29.7$\%$\\
Figure \ref{fig-G3-bc1}&
$P_l$&0.01$\%$&0.82$\%$& 0.96$\%$& 0.77$\%$& 76.6$\%$ &0.53$\%$& 7.42$\%$& 0.53$\%$& 0.56$\%$& 4.95$\%$&    6.77$\%$\\
\hline
\end{tabular}
\caption{The probabilities for the inputs in Figures \ref{fig-G2},
\ref{fig-G3},
\ref{fig-G3-bc},
\ref{fig-G3-bc1}
}
\label{Gt2}
\end{table}

\subsection{Defend adversaries}
\label{sec-lcae3}

We check the ability of LCAE to defend adversaries
for the three types of adversaries given in section \ref{sec-ad}.
The creation rates of adversaries for LCAE are given in the column ``LCAE'' of Table \ref{tab-ca2},
and the numbers of labels in the outputs of LCAE are given in column ``NLabels''.
From these tables, we can see that the  classification list  almost always
contains the correct labels for various adversaries by outputting about 3.5 labels.
In summary, the list classification can be considered to give an uncertain but lossless classification for these adversaries at the cots of giving about 3.5 labels from Table \ref{tab-ca2}.

In the following figures, we give three adversaries and the outputs of LCAE.
In Table \ref{Gt6}, we give the probabilities for these adversaries.
%
\begin{figure}[H]
\centering
\hspace{2mm}
\includegraphics[scale=0.54]{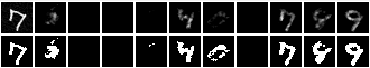}
\caption{A type 1.10 adversary of 7 (section \ref{sec-ad}); the correct label $7$ is given.
The first row is $x$ and $x_l$ in S2 of Algorithm \ref{alg2} and
the second row is $\widehat{x}$ and $\widehat{x}_l$, for $l=0,\ldots,9$.
}
\label{fig-G4}
\end{figure}

\begin{figure}[H]
\centering
\hspace{2mm}
\includegraphics[scale=0.55]{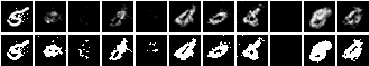}
\caption{A type 2.80 adversary of 5 (section \ref{sec-ad}); the wrong output label is 8.}
\label{fig-G71}
\end{figure}

\begin{figure}[H]
\centering
\hspace{2mm}
\includegraphics[scale=0.63]{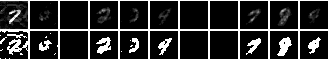}\\
\caption{A type 3.1 adversary of 7 (section \ref{sec-ad}); the wrong output label is $8$. }
\label{fig-Gi1}
\end{figure}


\begin{table}[H]\footnotesize
\centering
\begin{tabular}{|c|c|c|c|c|c|c|c|c|c|c|c|c|}
\hline
&$l=$ & $-1$ & $0$ & $1$ & $2$ & $3$ & $4$  & $5$ & $6$ & $7$ & $8$ & $9$ \\\hline
Figure \ref{fig-G4}&
 $P_l$&0.74$\%$&6.79$\%$& 3.45$\%$& 3.45$\%$& 3.67$\%$& 12.2$\%$& 7.34$\%$ &3.45$\%$& 27.2$\%$& 10.6$\%$& 20.9$\%$\\
Figure \ref{fig-G71}&
 $P_l$ &2.80$\%$&6.92$\%$& 3.26$\%$& 9.11$\%$& 3.03$\%$& 18.4$\%$& 7.18$\%$& 11.6$\%$& 2.67$\%$& 21.8$\%$&  13.1$\%$\\
Figure \ref{fig-Gi1}&
 $P_l$ &3.84$\%$&8.69$\%$& 4.69$\%$& 14.1$\%$& 8.86$\%$& 9.76$\%$& 4.69$\%$& 4.69$\%$& 14.9$\%$& 15.5$\%$& 10.1$\%$\\
\hline
\end{tabular}
\caption{Probabilities for the samples in Figures
\ref{fig-G4},
\ref{fig-G71},
\ref{fig-Gi1}.
 }
\label{Gt6}
\end{table}

\subsection{Recognize outliers}
\label{sec-lcae4}

In section \ref{sec-c}, it has already shown that CAE can recognize almost all outliers.
In this section, we give a more detailed analysis on the ability of $\F$ to defend outliers by giving the results of LCAE.
In Table \ref{GN101}, we give the percentages of the samples
whose $P_{-1}$ is lager than  $50\%$ and $80\%$, respectively.
From the table, we see that $P_{-1}>50\%$  for all images.
If a sample has $P_{-1}>50\%$, it is certain to be an outlier
and this explains the results in Table \ref{tab-co1} in more detail.

\begin{table}[H]
\centering
\begin{tabular}{|l|c|c|}
\hline
Outlier &  $P_{-1}>50\%$ & $P_{-1}>80\%$  \\\hline
Type 1  & 100$\%$& 100$\%$\\
Type 2& 100$\%$& 100$\%$\\
Type 3&100$\%$& 64$\%$\\
Type 4 &100$\%$&98$\%$\\
\hline
\end{tabular}
\caption{The percentages of outliers which have probability bigger than  $50\%$ and $80\%$}
\label{GN101}
\end{table}

In Figures \ref{fig-Gn-3} and \ref{fig-Gn-4}, two outliers in
Figure  \ref{fig-mnc} are given, respectively.
Their corresponding probabilities are given in Table \ref{Gt202}.

\begin{figure}[H]
\centering
\hspace{2mm}
\includegraphics[scale=0.55]{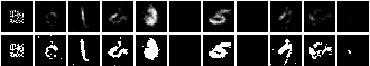}
\caption{A type 3 outlier from Figure \ref{fig-mnc} and the output from LCAE.}
\label{fig-Gn-3}
\end{figure}

\begin{figure}[H]
\centering
\hspace{2mm}
\includegraphics[scale=0.55]{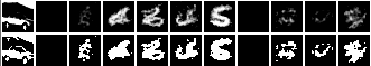}
\caption{A type 4 outlier in Figure \ref{fig-mnc} and the output from LCAE.}
\label{fig-Gn-4}
\end{figure}

\begin{table}[H]
\footnotesize
\centering
\begin{tabular}{|c|c|c|c|c|c|c|c|c|c|c|c|c|}
\hline
&$l=$ & $-1$ & $0$ & $1$ & $2$ & $3$ & $4$  & $5$ & $6$ & $7$ & $8$ & $9$ \\\hline
Figure \ref{fig-Gn-3}&
 $P_l$&79.5$\%$&2.20$\%$&2.44$\%$&2.17$\%$&2.20$\%$& 1.32$\%$&2.73$\%$&1.32$\%$&2.45$\%$& 2.29$\%$&1.36$\%$\\ \hline
 Figure \ref{fig-Gn-4}&
 $P_l$&99.0$\%$&0.01$\%$&0.01$\%$&0.01$\%$&0.01$\%$& 0.01$\%$&0.01$\%$&0.01$\%$&0.01$\%$& 0.01$\%$&0.01$\%$\\ \hline
\end{tabular}
\caption{The probabilities for the images in Figures \ref{fig-Gn-3} and \ref{fig-Gn-4}.}
\label{Gt202}
\end{table}

\subsection{Decouple-classification with LCAE}
\label{sec-lcae5}
Consider a special kind of outliers  obtained by ``adding'' two images from MNIST with different labels,
some of which are given in Figure \ref{fig-o3w}.
These images are not numbers, but they have strong characteristic of numbers.
\begin{figure}[H]
\centering
\hspace{2mm}
\includegraphics[scale=1.1]{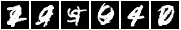}\\
\caption{Mixing of two numbers}
\label{fig-o3w}
\end{figure}


As an application of LCAE, we   give a {\em decouple-classification} algorithm which
can be used to find one or more of the elements in a sample
containing two or more well-mixed elements from $\O$.

We use 1000  outliers of this kind  as inputs to LCAE.
For about $54\%$ of these outliers,  the two numbers used to form the images are found by LCAE, and for $99\%$ of them,  one of the two numbers is found.

We give examples in Figures \ref{fig-Gnf1} and \ref{fig-Gnf2},
and the corresponding probabilities are given in Table \ref{tab-Gnf}.

\begin{figure}[H]
\centering
\hspace{2mm}
\includegraphics[scale=0.55]{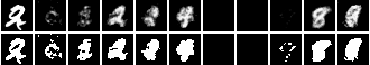}\\
\caption{LCAE finds the two number $2$ and $9$ used to form the image.}
\label{fig-Gnf1}
\end{figure}

\begin{figure}[H]
\centering
\hspace{2mm}
\includegraphics[scale=0.55]{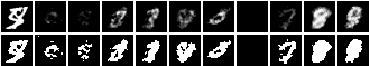}\\
\caption{LCAE finds one of the two number $4$ and $5$ used to form the image.}
\label{fig-Gnf2}
\end{figure}

\begin{table}[H]
\footnotesize
\centering
\begin{tabular}{|c|c|c|c|c|c|c|c|c|c|c|c|c|}
\hline
&$l=$ & $-1$ & $0$ & $1$ & $2$ & $3$ & $4$ &  $5$ & $6$ & $7$ & $8$ & $9$  \\\hline
Figure \ref{fig-Gnf1}&
$P_l$&3.07$\%$&3.25$\%$&7.39$\%$&20.4$\%$&11.4$\%$& 12.2$\%$ &1.73$\%$&1.73$\%$&2.14$\%$& 15.41$\%$&21.1$\%$\\
Figure \ref{fig-Gnf2}&
$P_l$&2.50$\%$&3.31$\%$&3.99$\%$&9.02$\%$&10.0$\%$&  12.3$\%$ &4.44$\%$&1.94$\%$&6.83$\%$& 21.9$\%$&23.6$\%$ \\
%
 %
\hline
\end{tabular}
\caption{The probabilities for the images in Figures \ref{fig-Gnf1} and \ref{fig-Gnf2}.}
\label{tab-Gnf}
\end{table}

\section{Conclusion}
\label{sec-conc}

In this paper, we consider the problem of building robust DNNs to defend
outliers and adversaries.
In the open-world classification problem, the inputs to the DNN
could be outliers which are irrelevant to the training set.
The objects to be classified  usually consist of
a low-dimensional subspace of the total input space to the DNN and the majority
of the input are outliers.
So, the DNN need to recognize outliers in order to be used in  open-world applications.
A more subtle robust problem is that, for almost all DNNs with moderate complex structures, there exist adversary samples.
The ability to defend adversaries is important for the DNN to be used in safety-critical applications.

In this paper, we present a new neural network structure which is more robust to recognize
outliers and to defend adversaries.
The basic idea is to change the autoencoder from an un-supervised learning
method to a classifier, where the encoder maps images with different labels
into disjoint compression subspaces and the decoder recovers the image from its
compression subspace.

The newly  introduced classification-autoencoder can recognize
almost all outliers due to the fact that the output of the autoencoder
is always similar the objects to be classified, and hence
achieves state of the art for outlier recognition.

Since adversaries are seemly inevitable for the current DNNs,
we introduce the list-classification based on the CAE, which outputs several labels instead of one.
According to our experiments, the list classifier can give
near lossless classification in the sense that the output list
contains the correct label for almost all adversaries
and the size of the output list is reasonably small.

An overall framework for robust classification could be done as follows.
First, CAE and LCAE are used to identify those inputs which are elements of $\O$ or outliers with high probabilities.
Second, for the remaining ``fuzzy'' samples, we may use LCAE to output
several labels and their corresponding recovered images for further analysis.



\begin{thebibliography}{99}

\bibitem{sur1}
 N. Akhtar and A. Mian.
 Threat of Adversarial Attacks on Deep Learning in Computer Vision: A Survey.
 arXiv:1801.00553v3, 2018.

\bibitem{asulay1}
 A. Azulay and Y. Weiss.
 Why Do Deep Convolutional Networks Generalize so Poorly to
 Small Image Transformations.
 {\em Journal of Machine Learning Research}, 20, 1-25, 2019.

\bibitem{sur2}
 T. Bai, J. Luo, J. Zhao.
 Recent Advances in Understanding Adversarial Robustness of Deep Neural Networks.
 ArXiv:2011.01539, 2020.

\bibitem{ae1}
 D.H. Ballard.
 Modular Learning in Neural Networks.
 {\em Proc. AAAI'87}, Vol. 1, 279-284, AAAI Press, 1987.

\bibitem{Bast1}
  A. Bastounis, A.C. Hansen, V.Vla$\breve{\rm{c}}$i$\acute{\rm{c}}$.
  The Mathematics of Adversarial Attacks in AI
  - Why Deep Learning is Unstable Despite the Existence of Stable Neural Networks.
  arXiv:2109.06098, 2021.


\bibitem{Biggio1}
 B. Biggio, I. Corona, D. Maiorca, B. Nelson, N. $\breve{\rm{S}}$rndi{\'c},
 P. Laskov, G. Giacinto, F. Roli.
 Evasion Attacks Against Machine Learning at Test Time.
 {\em Proc. of European Conference on Machine Learning and Knowledge Discovery in Databases},
 387-402, Springer, 2013.

\bibitem{obfu1}
 J. Buckman, A. Roy, C. Raffel, I. Goodfellow.
 Thermometer Encoding: One Hot Way to Resist Adversarial examples.
 {\em Proc. of the 6th International Conference on Learning Representations}, Vancouver, Canada, 2018.


\bibitem{outlier1}
 R. Chalapathy and S. Chawla.
 Deep Learning for Anomaly Detection: A Survey.
 ArXiv:1901.03407v2, 2019.


\bibitem{choi}
 C.Q. Choi.
 7 Revealing Ways AIs Fail.
 {\em IEEE Spectrum}, 42-47, October, 2021.

\bibitem{adv-li1}
 M. Cisse, P. Bojanowski, E. Grave, Y. Dauphin, N. Usunier.
 Parseval Networks: Improving Robustness to Adversarial Examples.
 {\em Proc. ICML'2017}, 854-863, 2017.

\bibitem{G1989}
 G. Cybenko.
 Approximation by Superpositions of a Sigmoidal Function.
 Mathematics of Control, Signals and Systems, 2(4): 303-314, 1989.

\bibitem{DL}
 I.J. Goodfellow, Y. Bengio,   A. Courville.
 {\em Deep Learning}, MIT Press, 2016.

\bibitem{G2014}
 I.J. Goodfellow, J. Shlens, C. Szegedy.
 Explaining and Harnessing Adversarial Examples.
 ArXiv:1412.6572, 2014.

\bibitem{obfu2}
  C. Guo, M. Rana, M. Cisse, L. van der Maaten.
  Countering Adversarial Images using Input Transformations.
  ArXiv: 1711.00117, 2017.

\bibitem{rob-b3}
 M. Hein, M. Andriushchenko.
 Formal Guarantees on the Robustness of a Classifier Against Adversarial Manipulation.
 {\em Proc. NIPS}, 2266-2276, 2017.

 \bibitem{H2015}
 G. Hinton, O. Vinyals, J. Dean.
 Distilling the Knowledge in a Neural Network.
 ArXiv:1503.02531, 2015.

\bibitem{uni1}
 K. Hornik.
 Approximation Capabilities of Multilayer Feedforward Networks.
 {\em Neural Networks}, 4(2): 251-257, 1991.

\bibitem{lecun2015deep}
 Y. LeCun, Y. Bengio, G. Hinton.
 Deep Learning.
 {\em Nature}, 521(7553), 436-444, 2015.

\bibitem{N2020}
 N. Lei, D. An, Y. Guo, K. Su, S. Liu, Z. Luo,  Z. Gu.
 A Geometric Understanding of Deep Learning.
 {\em Engineering}, 6(3), 361-374, 2020.

\bibitem{uni2}
 M. Leshno, V.Ya. Lin, A. Pinkus, and S. Schocken.
 Multilayer Feedforward Networks with a Nonpolynomial Activation
 Function Can Approximate any Function.
 {\em Neural Networks}, 6(6): 861-867, 1993.

\bibitem{W2019}
 W. Lin, Z. Yang, X. Chen, Q, Zhao, X. Li, Z. Liu, J. He.
 Robustness Verification of Classification Deep Neural Networks via Linear Programming.
 {\em CVPR'2019}, 11418-11427, 2019.

\bibitem{M2017}
 A. Madry, A. Makelov, L. Schmidt, D. Tsipras, A. Vladu.
 Towards Deep Learning Models Resistant to Adversarial Attacks.
 ArXiv:1706.06083, 2017.


\bibitem{voro}
 A. Okabe, B. Boots, K. Sugihara, S.N. Chiu, D.G. Kendall.
 Okabe, Michiko, Barry Boots, and Sung Nok Chiu.
 {\em Spatial Tessellations: Concepts and Applications of Voronoi Diagrams}.
 John Wiley \& Sons, New York, 2000.

\bibitem{P2016}
 N. Papernot, P. McDaniel, S. Jha, M. Fredrikson,  Z.B. Celik,  A. Swami,
 The Limitations of Deep Learning in Adversarial Settings.
 In 2016 IEEE European Symposium on Security and Privacy,
 372-387, IEEE Press, 2016.

\bibitem{uni3}
 A. Pinkus.
 Approximation Theory of the MLP Model in Neural Networks.
 {\em Acta Numerica}, 8: 143-195, 1999.

\bibitem{ae4}
 Y. Qi, Y. Wang, X. Zheng, Z. Wu.
 Robust Feature Learning by Stacked Autoencoder with Maximum Correntropy Criterion.
 {\em ICASSP'2014}, 6716-6720, IEEE Press, 2014.

\bibitem{rob-b1}
 A. Raghunathan, J. Steinhardt, P. Liang.
 Certified Defenses Against Adversarial Examples.
 ArXiv: 1801.09344, 2018.

\bibitem{ladder}
 A.V.H. Rasmus, M. Honkala, M. Berglund, T. Raiko.
 Semi-supervised Learning with Ladder Networks.
 {\em NIPS'15}, 3546--3554, 2015.

\bibitem{adv-inev1}
 A. Shafahi, W.R. Huang, C. Studer, S. Feizi, T. Goldstein.
 Are Adversarial Examples Inevitable?
 ArXiv:1809.02104, 2018.

\bibitem{adv-li3}
 A. Shafahi, M. Najibi, A. Ghiasi, Z. Xu, J. Dickerson, C. Studer,
 L.S. Davis, G. Taylor, T. Goldstein.
 Adversarial Training for Free!
 ArXiv: 1904.12843, 2019.


\bibitem{obfu3}
  P. Samangouei, M. Kabkab, R. Chellappa.
  Defense-GAN: Protecting Classifiers Against Adversarial Attacks using Generative Models.
  ArXiv: 1805.06605, 2018.


 \bibitem{outlier2}
 W.J. Scheirer, A. Rocha, A. Sapkota, T.E. Boult.
 Towards Open Set Recognition.
 {\em IEEE Trans. PAMI}, 36(7):1757-1772, 2013.

 \bibitem{obfu4}
 Y. Song, T. Kim, S. Nowozin, S. Ermon, N. Kushman.
 Pixeldefend: Leveraging Generative Models to Understand
 and Defend Against Adversarial Examples.
 ArXiv: 1710.10766, 2017.


\bibitem{uni4}
 M.H. Stone.
 The Generalized Weierstrass Approximation Theorem.
 {\em Mathematics Magazine}, 21(4): 167-184, 1948.

\bibitem{S2013}
 C. Szegedy, W. Zaremba, I. Sutskever, J. Bruna, D. Erhan, I.J. Goodfellow, R. Fergus.
 Intriguing Properties of Neural Networks.
 ArXiv:1312.6199, 2013.

\bibitem{adv-li2}
 F. Tramer, A. Kurakin, N. Papernot, I. Goodfellow, D. Boneh, P. McDaniel.
 Ensemble Adversarial Training: Attacks and Defenses.
 ArXiv: 1705.07204, 2017.

\bibitem{ae2}
 P. Vincent, H. Larochelle, Y. Bengio, P.A. Manzagol.
 Extracting and Composing Robust Features with Denoising Autoencoders.
 {\em Proc. ICML'08}, 1096-1103, ACM Press, 2008.

\bibitem{obfu5}
 C.H. Xie, J.Y. Wang, Z.S. Zhang, Z. Ren, A. Yuille.
 Mitigating Adversarial Effects Through Randomization.
 ArXiv: 1711.01991, 2017.

\bibitem{rob-b2}
 E. Wong, J. Z. Kolter.
 Provable Defenses Against Adversarial Examples via the Convex Outer Adversarial Polytope.
 ArXiv: 1711.00851, 2017.

\bibitem{sur-adv}
 H. Xu, Y. Ma, H.C. Liu, D, Deb, H. Liu J.L. Tang, A.K. Jain.
 Adversarial Attacks and Defenses in Images, Graphs and Text: A Review.
 {\em International Journal of Automation and Computing},
 17(2), 151-178, 2020.

\bibitem{Y2018}
 Z. Yang, X. Wang, Y. Zheng.
 Sparse Deep Neural Networks Using $L_{1,\infty}$-Weight Normalization.
 Statistica Sinica, 2020, doi:10.5705/ss.202018.0468.

\bibitem{L2i}
 L. Yu and X.~S. Gao.
 Improve the Robustness and Accuracy of Deep Neural Network
 with $L_{2,\infty}$ Normalization.
 ArXiv:2010.04912.

 \bibitem{bias}
 L. Yu and X.~S. Gao.
 Robust and Information-theoretically Safe Bias Classifier against Adversarial Attacks.
 arXiv:2111.04404, 2021.

\bibitem{YOPO}
 D.H. Zhang, T.Y. Zhang, Y.P. Lu, Z.X. Zhu, B. Dong.
 You Only Propagate Once: Accelerating Adversarial Training via Maximal Principle.
 ArXiv: 1905.00877, 2019.

\bibitem{rubust-rev1}
 X.Y. Zhang, C.L. Liu, C.Y. Suen.
 Towards Robust Pattern Recognition: A Review.
 {\em Proc. of the IEEE}, 108(6), 894-922, 2020.

\bibitem{ae3}
 C. Zhou and R.C. Paffenroth.
 Anomaly Detection with Robust Deep Autoencoders.
 {\em KDD'17}, 665-674, ACM Press, 2017.

\end{thebibliography}

\vskip20pt
\section*{Appendix A. Proof of Theorem 1}
First introduce the notion of Voronoi tessellation.
 Let  $\I_0=(0,1)\subset\R$,
$A\subset \I_0^n$ a convex open set, and $P=\{p_i,i=1,\ldots,t\}\subset A$.
For each $p_i$, let $R_i$ be the set of points in $A$, which are strictly closer to $p_i$ than to
$p_j,j\ne i$.
Then $R_i$ is an open convex set
and $\{R_i\}_{i=1}^{t}$
are called the {\em Voronoi tessellation} of $A$ generated by $P$,
and $R_i$ is called the Voronoi region with generating point $p_i$~\cite{voro}.
It is clear that
$\overline{A}=\cup_{i=1}^t (\overline{R}_i)$,
where $\overline{R}_i$ is the closure of $R_i$.

We first consider the case where the objects to be classified
consist of a single connected open set.

\begin{lemma}
\label{th-1}
Let $S$ be a bounded  open set in $\R^n$ and $A$ a bounded and  convex open set in $\R^m$.
%
%
For any $\epsilon, \gamma\in\R_{+}$ and any $m\in\N_{+}$,  there exist functions
$E:\I^n\to\R^m$ and $D:\R^m\to\I^n$, which are piecewise continuous functions with a finite number of continuous regions and satisfy
\begin{description}
\item[0.] $E(x)=0$ if $x\not\in S$ and $D(y)=0$ if $y\not\in A$.
\item[1.]
$V_1=\{x\in S\subset \R^n\,|\, E(x)\in A\}$ satisfies $\frac{V({V}_1)}{V(S)}>1-\epsilon$;
\item[2.]
$V_2=\{y\in A\subset \R^m\,|\,D(y)\in S\}$ satisfies $\frac{V({V}_2)}{V(A)}>1-\epsilon$;
\item[3.]
$V_3=\{x\in S\subset\R^n\,|\,||x-D(E(x))||<\gamma\}$ satisfies $\frac{V({V}_3)}{V(S)}>1-\epsilon$.
\end{description}
\end{lemma}
\begin{proof}
Without loss of generality, assume $S\subset\I_0^n$.
%
Let $k\in\N_{+}$ and
$$C^{k, d_1, d_2, \dots,  d_n}=(\frac{d_1}{k}, \frac{d_1+1}{k})
\times(\frac{d_2}{k}, \frac{d_2+1}{k})\dots\times(\frac{d_n}{k}, \frac{d_n+1}{k}),$$  where $d_i\in\{0, 1, 2, \dots, k-1\}$.
Let
$${S}_k=\{C^{k, d_1, d_2, \dots,  d_n}|C^{k, d_1, d_2, \dots,  d_n}\subset S \}.$$
It is easy to see that $V({S}_k) = t_k/k^n$, where $t_k$ is the number of cubes in ${S}_k$. When $k$ becomes lager, $V({S}_k)$  will increase and approach to $V(S)$.
Then, we can choose a $k$ such that
%
\begin{equation}
\label{eq-Sk1}
k>\frac{\sqrt{m}}{\gamma}
\hbox{ and }
\frac{V({S}_k)}{V(S)}=\frac{\frac{t_k}{k^n}}{V(S)}>1-\epsilon.
\end{equation}
%
%
For simplicity, let ${S}_k=\{C_i\}_{i=1}^{t_k}$.
%
%
Let
\begin{equation}
\label{eq-Ap}
A_{k}=\{a_i\}^{t_{k}}_{i=1}
\end{equation}
 be $t_{k}$ distinct points in $A$.
Let $R=\{R_i\}^{t_{k}}_{i=1}$ be the Voronoi tessellation of $A$ generated by $A_k$
and $a_i$ the generating point for $R_i$.
Define $E$ and $D$ as follows
\begin{eqnarray}
E(x)&:&\I^n\to\R^m, E(x)=a_i\hbox{ if }x\in C_i\hbox{ and }E(x)=0\hbox{ otherwise.}\\
D(y)&:& \R^m\to\I^{n},D(y)= {c}_i\hbox{ if }y\in R_i\nonumber
\end{eqnarray}
where $ {c}_i$ is the center of $C_i$.
%
It is clear that $E(x)$ is a constant function over each $C_i$
and $D(x)$ is a constant function over each $R_i$.

We now prove that $E$ and $D$ satisfy the properties of the lemma.
From the above construction, we have
$V_1=\cup_{i=1}^{t_k} C_i$ and $V(V_1) = V(S_k)$.
Then property 1 follows from \eqref{eq-Sk1}.
It is easy to see $V_2=\cup_{i=1}^{t_k} R_i$
and $\overline{V}_2 = A$.  Then  $\frac{V({V_2})}{V(A)}=1$, and property 2 is proved.
%
 %
For $C_i\in {S}_{k}$, if $x\in C_i$, then  $D(E(x))$ is the center of $C_i$, and hence $||D(E(x))-x||<\frac{\sqrt{m}}{k}<\gamma$
by \eqref{eq-Sk1}.
Then $V_3 = V_1$ and the lemma is proved.
\end{proof}

From the proof of Lemma \ref{th-1}, we have
\begin{cor}
\label{thcor-1}
Use the notations in Lemma \ref{th-1}. There exists a number $t$ such that
\begin{description}
\item[1.] $V_1=V_3=\cup_{i=1}^{t} C_i$, where $C_i\subset S$ are disjoint open cubes with center $c_i$.

\item[2.] Let $A_t=\{a_i\}_{i=1}^t$ be $t$ distinct points inside $A$
and $R_i,i=1,\ldots,t$ the Voronoi polyhedra generated by $A_t$ in  $A$
and $a_i\in R_i$.
Then  $V_2=\cup_{i=1}^{t} R_i$ and $A=\overline{V}_2$.

\item[3.]
$E(x) = a_i$ for $x \in C_i$ and $E(x) = 0$ otherwise.
$D(y) = c_i$ for $y \in R_i$ and $D(y) = 0$ otherwise.
\end{description}
\end{cor}

For a set $W\subset\R^n$ and $a\in\R_{+}$, define $W^{a}=\{x\in W\,|\,\exists  r>a, \hbox{ s.t. } \B(x,r)\subset W\}$.

\begin{lemma}
\label{l-1}
%
Let $F:\I_0^n\to\R^m$ be a piecewise linear function with a finite number of linear  regions. Then for any $\epsilon>0$, $\gamma>0$, $\alpha>0$, there exists a DNN $\F:\I_0^n\to\R^m$ such that
$|\F(x)-F(x)|<\epsilon$ for $x\in A^{\gamma}$, where $A$ is any linear region of $F$.
Moreover, $D=\{x \,|\, \F(x)-F(x)|>\epsilon\}$ satisfies $V(D)<\alpha$.
\end{lemma}
\begin{proof}
It is easy to see that there exists a continuous function $H$ which satisfies that $|F(x)-H(x)|=0$ for $x\in A_{\gamma}$.
Then the lemma follows from the universal  approximation theorem of DNN~\cite{uni1,uni2,uni3}.
\end{proof}

%

Now we give the proof of Theorem \ref{th-3}.

\begin{proof}
Let $A_0=\I_0^m$ and $A_j=\{y\,|\,y=z+2j\one, z\in A_{j-1}\}$ for $j=1,\dots,o$, where $\one\in \R^m$ is the vector whose coordinates are $1$.
We first assume that each $S_l$ is connected.
By Lemma \ref{th-1} and Corollary \ref{thcor-1}, for each $l\in\L$ and $\epsilon_1,\gamma\in\R_{+}$,
there exist functions $E_l:\R^n\rightarrow \R^m$ and $D_l:\R^m\rightarrow \R^n$,  which are piecewise continuous functions with a finite number of  continuous regions and satisfy

\begin{description}
\item[(A1)]
$V_{l,1}=\cup_{i=1}^{t} C_{l,i}$, where $C_{l,i}\subset S_l$ are disjoint open cubes with center $c_{l,i}$.
Furthermore,  $\frac{V(V_{l,1})}{V(S_l)}>1-\epsilon_1$
and  $||x-D_l(E_l(x))||<\gamma$ for $x\in V_{l,1}$.

\item[(A2)]
$U_l=\{a_{l,i}\}_{i=1}^t$ is a set of $t$ distinct points inside $A_l$
and $R_{l,i},i=1,\ldots,t$ the Voronoi polyhedra generated by $U_l$ inside $A_l$
and $a_{l,i}\in R_{l,i}$. Therefore, $A_l = \cup_{i=0}^o \overline{R}_{l,i}$.

\item[(A3)] $E_l(x)$ and $D_l(y)$ are defined as follows:
$E_l(x) = a_{l,i}$ for $x \in C_{l,i}$ and $E_l(x) = 0$ otherwise;
$D_l(y) = c_{l,i}$ for $y \in R_{l,i}$ and $D_l(y) = 0$ otherwise.

\end{description}
%
Define $E$ and $D$ as follows: $E(x)=\sum_{i=0}^{o}E_i(x)$ and
$D(y)=\sum_{i=0}^{o}D_i(y)$.
%
%
%
Note that $E$ is constant over $C_{l,i}$ and $D$ is  constant over $R_{l,j}$.
We call $C_{l,i}$ and $R_{l,j}$ constant regions of $E$ and $D$, respectively.

By Lemma \ref{l-1}, for any $\epsilon_2$, $\epsilon_3$, and $\beta$, there exist DNNs $\E$ and $\D$ such that
\begin{description}
\item[(B1)]
$||\E(x)-E(x)||<\epsilon_2$, for $x\in \I_0^n \setminus S_{-1}$, where $S_{-1}\subset \I_0^n$,
$V(S_{-1})<\epsilon_3$, and $C_{l,i}^{\beta/4}\bigcap S_{-1}=\emptyset$
for any constant region $C_{l,i}$ of $E$.
\item[(B2)]
 $||\D(y)-D(y)||<\epsilon_2$, for $y\in A\setminus S_{-2}$, where
$A=\cup_{i=0}^o A_i$,
$S_{-2}\subset A$,
$V(S_{-2})<\epsilon_3$,
and $R_{l,i}^{\beta/4}\bigcap S_{-2}=\emptyset$ for any constant region $R_{l,i}$ of $D$.
\end{description}

Let $\beta_l$ be the minimum of the distances between any pair of points in $E(S_l)=U_l=\{a_{l,i}\}_{i=1}^t$
and the distances between any point in $E(S_l)$ and the surface of $A_l$.
%
%
Choose the parameters such that
\begin{description}
\item[(C0)]
 $\beta < \beta_l$ for all $l\in\L$.
\item[(C1)]
 $\epsilon_2\le\beta/4$ and $\epsilon_2\le1/2$.
\item[(C2)]
 $\epsilon_3<\epsilon \sum_{l}V(S_l)$.
\item[(C3)]
$\gamma+\epsilon_2<\gamma$, $\epsilon_1+\frac{\epsilon_3}{V(S_l)}<\epsilon$ for all $l$.
\end{description}

We now prove that $\E$ and $\D$ satisfy the properties in the theorem.
Let $x\in K$.
Then  $\E(x)\in\E(S_i)\cap\E(S_j)$  for $i\ne j$.
By property (B1), if $x\in S_i/S_{-1}$ then
$\E(x)\in D_i=(-\epsilon_2,1+\epsilon_2)^m+2i\one$.
Similarly, if $x\in S_j/S_{-1}$ then
$\E(x)\in D_j=(-\epsilon_2,1+\epsilon_2)^m+2j\one$.
By (C1), we have $D_i\bigcap D_j=\emptyset$ if $i\ne j$. So $x\in K$ implies $x\in S_{-1}$, and hence $V(K)<\epsilon_3$  by (B1).
Because of (C2), we have $\frac{V(K)}{\sum_{l}V(S_l)}<\epsilon$.
The first property of the theorem is proved.

Let $x\in C_{l,i}/S_{-1}$. We will prove
\begin{equation}
\label{eq-T1}
||\D(\E(x))-x||\le||D(\E(x))-x||+\epsilon_2\le\gamma+\epsilon_2\le\gamma.
\end{equation}

The first inequality in \eqref{eq-T1} follows from
$||\D(\E(x))-D(\E(x))||<\epsilon_2$ which will be proved below.
By (A3),  $E(x)=a_{l,i}\in R_{l,i}$  for  $x\in C_{l,i}$.
For any $x_0\in \B(E(x),\beta/2)$, $E(x)$ is the nearest point of $x_0$ in $E(S_l)$ by (C0),
and hence $\B(E(x),\beta/2)\subset R_{l,i}$ is a constant region of $D(y)$ by (A2).
By (C1), $\B(E(x),\epsilon_2) \subset \B(E(x),\beta/2)\subset R_{l,i}$ and hence $\B(E(x),\epsilon_2)\subset R_{l,i}^{\beta/4}$.
By (B1), $\E(x)\in \B(E(x),\epsilon_2)$ since $x\in C_{l,i}/S_{-1}$, hence  $\E(x)\in R_{l,i}^{\beta/4}$.
By (B2), $\E(x)\not\in S_{-2}$ and hence  $||\D(\E(x))-D(\E(x))||<\epsilon_2$.
The first inequality in \eqref{eq-T1} is proved.

We now prove the second inequality in \eqref{eq-T1}.
We already proved  $\E(x)\in \B(E(x),\epsilon_2)\subset R_{l,i}$
and hence  $D(\E(x))=D(E(x))$ by (A3).
By (A1), $||D(\E(x))-x|| = ||D(E(x))-x|| < \gamma$
since $x\in C_{l,i}\subset V_{l,1}$.
The last inequality in \eqref{eq-T1} follows from (C3).

By \eqref{eq-T1}, $V_{l,1}/S_{-1}\subset \widehat{S}_l$.
Finally, by (B1), (A1),  and (C3),
we have,
$\frac{V(\widehat{S}_{l})}{V(S_{l})} \ge
\frac{V(V_{l,1}/S_{-1})}{V(S_{l})}>
\frac{V(V_{l,1}) -\epsilon_3 }{V(S_{l})}
>(1-\epsilon_1)-\frac{\epsilon_3}{V(S_{l})}>1-\epsilon.$
The theorem is proved.

In the above proof, each $S_l$ is assumed to be connected.
It is clear that the proof can be easily modified for the case such that different $S_l$ have the same label.
\end{proof}

\section*{Appendix B. Structure of the network }

We give the structure of the networks  used in the experiments.
We first give the structure of the CAE.

{\bf The structure of $\E$:}

Input layer: $N\times 1\times  28\times  28$,  where $N$ is steps of training.

Hidden layer 1: a convolution layer with kernel $1\times 10\times 3\times 3$ with padding$=1$ $\to$ do a batch normalization $\to$ do Relu $\to$ use max pooling with step=2.

Hidden layer 2: a convolution layer  with kernel $10\times 28\times 3\times 3$ with padding$=1$ $\to$ do a batch normalization $\to$ do Relu $\to$ use max pooling with step=2.

Hidden layer 3: a convolution layer  with kernel $28\times 28\times 3\times 3$ with padding$=1$ $\to$ do a batch normalization $\to$ do Relu $\to$ use max pooling with step=2.

Hidden layer 4: draw the output as $N\times 252$ $\to$ use a full connection with output size $N\times 168$ $\to$ do Relu.

Output layer: a full connection layer with output size $N\times 100$ $\to$ do Relu.

\vskip10pt
{\bf The structure of $\D$:}

Input layer: $N\times 100$

Hidden layer 1: a full connection layer with output size $N\times 252$ $\to$ do Relu.

Hidden layer 2: a full connection layer with output size $N\times 600$ $\to$ do Relu.

Hidden layer 3: a full connection layer with output size $N\times 2700$ $\to$ do Relu.

Hidden layer 4: draw the output as $N\times 3\times 30\times 30$ $\to$ use convolution with kernel $3\times 28\times 3\times 3$ with padding$=0$ $\to$ do a batch normalization $\to$ do Relu.

Output layer: use convolution with kernel $28\times 1\times 1\times 1$ with padding$=0$ $\to$ do Relu.

\vskip10pt
In section \ref{sec-out}, a CNN $\H$  is used to recognize outliers.
In section \ref{sec-ad}, a CNN trained with adversarial training is used for comparison.
These two CNNs have the following  structure.

{\bf Input layer}: $N\times 1\times 28\times 28$.

Hidden layer 1: a convolution layer  with kernel $1\times 32\times 3\times 3$ with padding$=1$  $\to$ do Relu $\to$ use max pooling with step=2.

Hidden layer 2: a convolution layer  with kernel $32\times 64\times 3\times 3$ with padding$=1$  $\to$ do Relu $\to$ use max pooling with step=2.

Hidden layer 3: a convolution layer  with kernel $64\times 64\times 3\times 3$ with padding$=1$ $\to$ do Relu $\to$ use max pooling with step=2.

Hidden layer 4: draw the output as $N\times 576$ $\to$ use a full connection layer with output size $N\times 128$ $\to$ do Relu.

{\bf Output layer}: a full connection layer with output size $N\times 10$.
The output layer of $\H$ is a full connection layer with output size $N\times 11$

\end{document}